\documentclass[11pt]{article}
\usepackage{soul}
\usepackage{bm,amsmath,amssymb,amsthm}
\usepackage{custom_style}
\usepackage[ruled,vlined,linesnumbered]{algorithm2e}
\usepackage{graphicx}
\usepackage{hyperref}
\usepackage{url}
\usepackage{fullpage}
\usepackage{xcolor}
\usepackage{xspace}
\usepackage{enumitem}
\usepackage{booktabs}
\usepackage{multirow}
\usepackage[normalem]{ulem}
\usepackage{subfigure}
\useunder{\uline}{\ul}{}
\definecolor{mydarkblue}{rgb}{0,0.08,0.45}
\hypersetup{ 
    colorlinks,
    citecolor=mydarkblue,
    filecolor=mydarkblue,
    linkcolor=mydarkblue,
    urlcolor=mydarkblue 
}
\usepackage{array}

\newcolumntype{L}[1]{>{\raggedright\let\newline\\\arraybackslash\hspace{0pt}}m{#1}}
\newcolumntype{C}[1]{>{\centering\let\newline\\\arraybackslash\hspace{0pt}}m{#1}}
\newcolumntype{R}[1]{>{\raggedleft\let\newline\\\arraybackslash\hspace{0pt}}m{#1}}

\newcommand{\argmax}{\operatornamewithlimits{argmax}}
\newcommand{\argmin}{\operatornamewithlimits{argmin}}

\let\oldnl\nl
\newcommand{\nonl}{\renewcommand{\nl}{\let\nl\oldnl}}

\SetCommentSty{mycommfont}
\SetNlSty{emph}{}{}

\newtheorem{prop}{\textbf{Proposition}}

\newtheorem{problem}{\textbf{Problem}}

\begin{document}

\title{\textbf{FLASH: Fast Bayesian Optimization for\\ Data Analytic Pipelines}
\vspace{10pt}
\date{}
\author{
Yuyu Zhang \quad Mohammad Taha Bahadori \quad Hang Su \quad Jimeng Sun\\
Georgia Institute of Technology\\
{\fontfamily{pcr}\selectfont\{yuyu,bahadori,hangsu\}@gatech.edu,jsun@cc.gatech.edu}\\
}
}

\maketitle
\begin{abstract} 
Modern data science relies on data analytic pipelines to organize interdependent computational steps. Such  analytic pipelines often involve different algorithms across multiple steps, each with its own hyperparameters. To achieve the best performance, it is often critical to select optimal algorithms and to set appropriate hyperparameters, which requires large computational efforts. 
Bayesian optimization provides a principled way for searching optimal hyperparameters for a single algorithm. However, many challenges remain in solving pipeline optimization problems with high-dimensional and highly conditional search space. In this work, we propose Fast LineAr SearcH (FLASH), an efficient method for tuning analytic pipelines. FLASH is a two-layer Bayesian optimization framework, which firstly uses a parametric model to select promising algorithms, then computes a nonparametric model to fine-tune hyperparameters of the promising algorithms. FLASH also includes an effective caching algorithm which can further accelerate the search process.
Extensive experiments on a number of benchmark datasets have demonstrated that FLASH significantly outperforms previous state-of-the-art methods in both search speed and accuracy. Using 50\% of the time budget, FLASH achieves up to 20\% improvement on test error rate compared to the baselines.
FLASH also yields state-of-the-art performance on a real-world application for healthcare predictive modeling.
 
\end{abstract}

%
%

%


\section{Introduction}
\label{sec:intro}
Modern data science often requires many computational steps such as data preprocessing, feature extraction, model building, and model evaluation, all connected in a \textit{data analytic pipeline}. Pipelines provide a natural way to represent, organize and standardize data analytic tasks, which are considered to be an essential element in the data science field~\cite{donoho2015} due to their key role in large-scale data science projects. Many machine learning toolboxes such as scikit-learn \cite{pedregosa2011scikit}, RapidMiner \cite{mierswa2006yale}, SPSS \cite{coakes2009spss}, Apache Spark \cite{mengml} provide mechanisms for configuring analytic pipelines.

\begin{figure*}[t]
\centering
\includegraphics[width=1\textwidth]{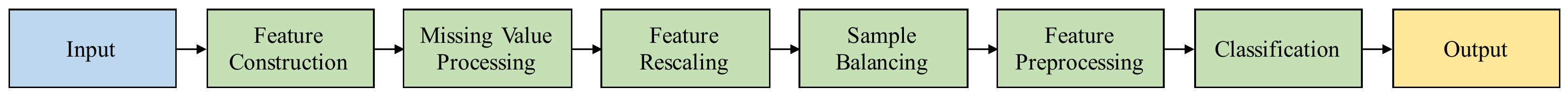}
\caption{A typical data analytic pipeline.}
\label{fig:concept_pipeline}
\end{figure*}

An analytic pipeline skeleton is shown in Figure~\ref{fig:concept_pipeline}. Each step, such as feature preprocessing and classification, includes many algorithms to choose from. These algorithms usually require users to set hyperparameters, ranging from optimization hyperparameters such as learning rate and regularization coefficients, to model design hyperparameters such as the number of trees in random forest and the number of hidden layers in neural networks. There are an exponential number of choices for the combination of algorithms and hyperparameters in a given analytic pipeline skeleton. 
Because of the interdependency between all the algorithms and their hyperparameters, the choices can have huge impact on the performance of the best model. 

Tuning hyperparameters of a single algorithm can be viewed as an optimization problem of a black-box objective function, which is noisy and often expensive to evaluate. Here the input of black-box are the hyperparameters, and the objective function is the output performance such as accuracy, precision and recall. To tackle this problem,  simple methods have been applied such as grid or random search \cite{bergstra2011algorithms,bergstra2012random}. While on difficult problems where these simple approaches are not efficient, a more promising model-based approach is Bayesian optimization \cite{mockus1978application,lizotte2008practical,brochu2010tutorial,shahriari2016taking}. The high-level idea of Bayesian optimization is to define a relatively cheap surrogate function and use that to search the hyperparameter space. Indeed, there exist other global optimization methods, such as evolutionary algorithms \cite{back1996evolutionary} and optimistic optimization \cite{munos2011optimistic}. We choose Bayesian optimization framework due to its great performance in practice. Recently, Bayesian optimization methods have been shown to outperform other methods on various tasks, and in some cases even beat human domain experts to achieve better performance via tuning hyperparameters ~\cite{snoek2015scalable,bergstra2013making}.

Despite its success, applying Bayesian optimization for tuning analytic pipelines faces several significant challenges:  Existing Bayesian optimization methods are usually based on nonparametric models, such as Gaussian process and random forest. A major drawback of these methods is that they require a large number of observations to find reasonable solutions in high-dimensional space. When tuning a single algorithm with several hyperparameters, Bayesian optimization works well with just a few observations. However, when it comes to pipeline tuning, thousands of possible combinations of algorithms plus their hyperparameters jointly create a large hierarchical high-dimensional space to search over, whereas existing methods tend to become inefficient. Wang et al.~\cite{wang2013bayesian} tackled the high-dimensional problem by making a low effective dimensional assumption. However, it is still a flat Bayesian optimization method and not able to handle the exploding dimensionality problem caused by hierarchically structured hyperparameters in analytic pipeline tuning.

\noindent{\bf Motivating example:} We build an analytic pipeline for classification task (details  in Section~\ref{sec:exp}). If we give 10 trials for each hyperparameter over 1,456 unique pipeline paths and 102 hyperparameters, we have more than 2 million configurations, which can take years to complete with a brute-force search. Even with the state-of-the-art Bayesian optimization algorithm such as Sequential Model-based Algorithm Configuration (SMAC)~\cite{hutter2011sequential}, the process can still be slow as shown in Figure~\ref{fig:gap}. If we know the optimal algorithms ahead time (Oracle) with just hyperparameter tuning of the optimal algorithms, we can obtain significant time saving, which is however not possible. Finally, our proposed method FLASH can converge towards the oracle performance much more quickly than SMAC.

In this paper, we propose a two-layer Bayesian optimization algorithm called Fast LineAr SearcH (FLASH): the first layer for selecting algorithms, and the second layer for tuning the hyperparameters of selected algorithms. 
FLASH is able to outperform the state-of-the-art Bayesian optimization algorithms by a large margin, as shown in Figure~\ref{fig:gap}. By designing FLASH, we make three main contributions:

\begin{figure}[t]
\centering
\includegraphics[width=0.6\textwidth]{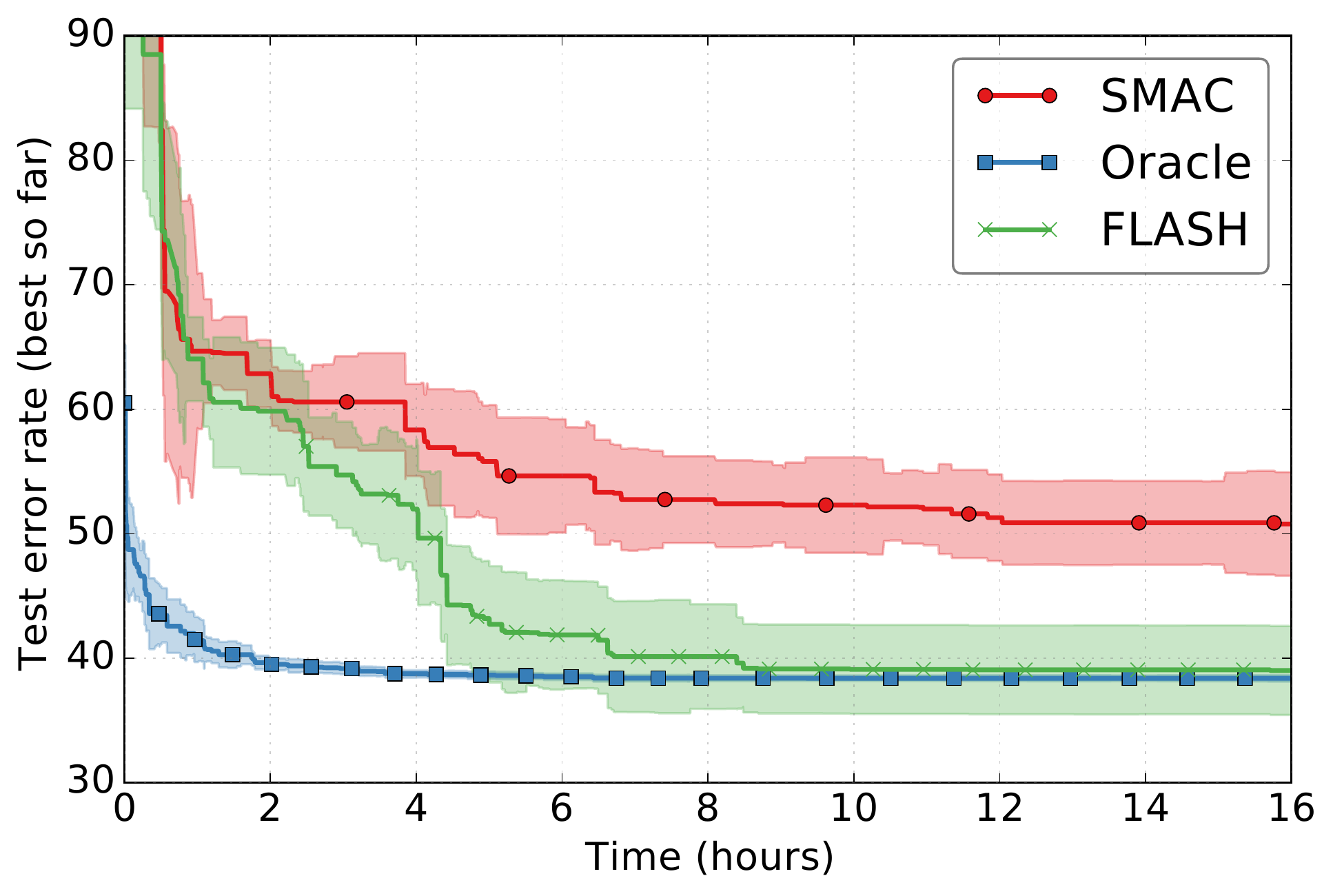}
\caption{Performance comparison of a data analytic pipeline on MRBI dataset, including the previous state of the art SMAC, proposed method FLASH, and Oracle (i.e., pretending the optimal algorithm configuration is given, and only performing hyperparameter tuning with SMAC on those algorithms). We show the median percent error rate on the test set along with standard error bars (generated by 10 independent runs) over time. FLASH outperforms SMAC by a big margin and converges toward Oracle performance quickly.}
\label{fig:gap}
\end{figure}

\begin{itemize}[leftmargin=0.15in]
\item We propose a linear model for propagation of error (or other quantitative metrics) in analytic pipelines. We also propose a Bayesian optimization algorithm for minimizing the aggregated error using our linear error model. Our proposed mechanism can be considered as a hybrid model: a parametric linear model for fast exploration and pruning of the algorithm space, followed by a nonparametric hyperparameter fine-tuning algorithm.
\item We propose to initialize the hyperparameter tuning algorithm using the optimal design strategy \cite{atkinson2012optimum,flaherty2005robust,Settles2012} which is more robust than the random initialization. We also propose a fast greedy algorithm to efficiently solve the optimal design problem for any given analytic pipeline.
\item Finally, we introduce a caching algorithm that can significantly accelerate the tuning process. In particular, we model the (time) cost of each algorithm, and incorporate that in the optimization process. This ensures the efficiency of fast search.
\end{itemize}
 
We demonstrate the effectiveness of FLASH with extensive experiments on a number of difficult problems. On the benchmark datasets for pipeline configurations tuning, FLASH substantially improves the previous state of the art by 7\% to 25\% in test error rate within the same time budget. We also experiment with large-scale real-world datasets on healthcare data analytic tasks where FLASH also exhibits superior results.

\section{Background and Related Work}
\label{sec:back}
\subsection{Data Analytic Pipelines}

The data analytic pipeline refers to a framework consisting of a sequence of computational transformations on the data to produce the final predictions (or outputs) \cite{kumar2015model}.  Pipelines help users better understand and organize the analysis task, as well as increase the reusability of algorithm implementations in each step.
Several existing widely adopted machine learning toolboxes provide the functionality to run analytic pipelines. Scikit-learn \cite{pedregosa2011scikit} and Spark ML \cite{mengml} provide programmatic ways to instantiate a pipeline. SPSS \cite{coakes2009spss} and RapidMiner \cite{mierswa2006yale} provide a visual way to assemble an analytic pipeline instance together and run. Microsoft Azure Machine Learning\footnote{\url{https://studio.azureml.net}} provides a similar capability in a cloud setting. There are also specialized pipelines, such as PARAMO \cite{ng2014paramo} in healthcare data analysis.

However, a major difficulty in using these systems is that none of the above described tools is able to efficiently help users decide which algorithms to use in each step. Some of the tools such as scikit-learn, Spark ML, and PARAMO allow searching all possible pipeline paths and tuning the hyperparameters of each step using an expensive grid search approach. While the search process can be sped up by running in parallel, the search space is still too large for the exhaustive search algorithms.

\subsection{Bayesian Optimization}

Bayesian optimization is a well-established technique for global and black-box optimization problems. In a nutshell, it comprises two main components: a probabilistic model and an acquisition function. For the probabilistic model, there are several popular choices: Gaussian process \cite{snoek2012practical,snoek2015scalable}, random forest such as Sequential Model-based Algorithm Configuration (SMAC) \cite{hutter2011sequential}, and density estimation models such as Tree-structured Parzen Estimator (TPE) \cite{bergstra2011algorithms}. Given any of these models, the posterior mean and variance of a new input can be computed, and used for computation of the acquisition function. The acquisition function defines the criterion to determine future input candidates for evaluation. Compared to the objective function, the acquisition function is chosen to be relatively cheap to evaluate, so that the most promising next input for querying can be found quickly. Various forms of acquisition functions have been proposed \cite{srinivas2009gaussian,hoffman2011portfolio,villemonteix2009informational,hoffman2014correlation}. One of the most prominent acquisition function is the \emph{Expected Improvement} (EI) function \cite{mockus1978application}, which has been widely used in Bayesian optimization. In this work, we use EI as our acquisition function, which is formally described in Section~\ref{sec:method}.

Bayesian optimization is known to be successful in tuning hyperparameters for various learning algorithms on different types of tasks \cite{snoek2015scalable,feurer2015initializing,bergstra2013making,snoek2012practical,wang2013bayesian}. Recently, for the problem of pipeline configurations tuning, several Bayesian optimization based systems have been proposed: Auto-WEKA \cite{thornton2013auto} which applies SMAC \cite{hutter2011sequential} to WEKA \cite{hall2009weka}, auto-sklearn \cite{feurer2015efficient} which applies SMAC to scikit-learn \cite{pedregosa2011scikit}, and hyperopt-sklearn \cite{komer2014hyperoptsklearn} which applies TPE \cite{bergstra2011algorithms} to scikit-learn. The basic idea of applying Bayesian optimization to pipeline tuning is to expand the hyperparameters of all algorithms and create large search space to perform optimization as we will show in the experiments. However, for practical pipelines the space becomes too large which hinders convergence of the optimization process. Auto-sklearn \cite{feurer2015efficient} uses a meta-learning algorithm that leverages performance history of algorithms on existing datasets to reduce the search space. However, in real-world applications, we often have unique datasets and tasks such that finding similar datasets and problems for the meta-learning algorithm will be difficult.

\section{Methodology}
\label{sec:method}
A data analytic pipeline $G = (V, E)$ can be represented as a multi-step Directed Acyclic Graph (DAG), where $V$ is the set of algorithms, and $E$ is the set of directed edges indicating dependency between algorithms. Algorithms are distributed among multiple steps. Let $V^{(k)}_i$ denote the $i$th algorithm in the $k$th step. Each directed edge $(V^{(k)}_i, V^{(k+1)}_j) \in E$ represents the connection from algorithm $V^{(k)}_i$ to $V^{(k+1)}_j$. Note that there is no edge between algorithms in the same step.
We also have an input data vertex $V_{in}$ which points to all algorithms in the first step, and an output vertex $V_{out}$ which is pointed by all algorithms in the last step. 

A \textit{pipeline path} is any path from the input $V_{in}$ to the output $V_{out}$ in pipeline graph $G$ . To denote a pipeline path of $K$ steps, we use $K$ one-hot vectors $\bm{p}^{(k)}$ ($1 \leq k\leq K$), each denoting the algorithm selected in the $k$-th step. Thus, the concatenation of one-hot vectors $\bm{p} = \big[ \bm{p}^{(1)}, \ldots, \bm{p}^{(K)} \big] \in \{0,1\}^N$ denotes a pipeline path, where $N$ is the total number of algorithms in the pipeline $G$. Figure~\ref{fig:sample_pipeline} shows a small data analytic pipeline with two steps. The first step contains two algorithms, and the second step contains three. One possible pipeline path is highlighted in the shaded area. On this pipeline path, $V^{(1)}_2$ and $V^{(2)}_3$ are selected in the first and second step, so that we have $\bm{p}^{(1)} = [0, 1]$ and $\bm{p}^{(2)} = [0, 0, 1]$. Thereby, the highlighted pipeline path is given by $\bm{p} = \big[ \bm{p}^{(1)}, \bm{p}^{(2)} \big] = [0, 1, 0, 0, 1]$. For any pipeline path $\bm{p}$, we concatenate all of its hyperparameters in a vector $\bm{\lambda}_{\bm{p}}$. The pair of path and hyperparameters, i.e. $(\bm{p}, \bm{\lambda}_{\bm{p}})$, forms a pipeline configuration to be run. For ease of reference, we list the notations in Table~\ref{table:notation}.
      
\begin{figure}[t]
\centering
\includegraphics[width=0.6\textwidth]{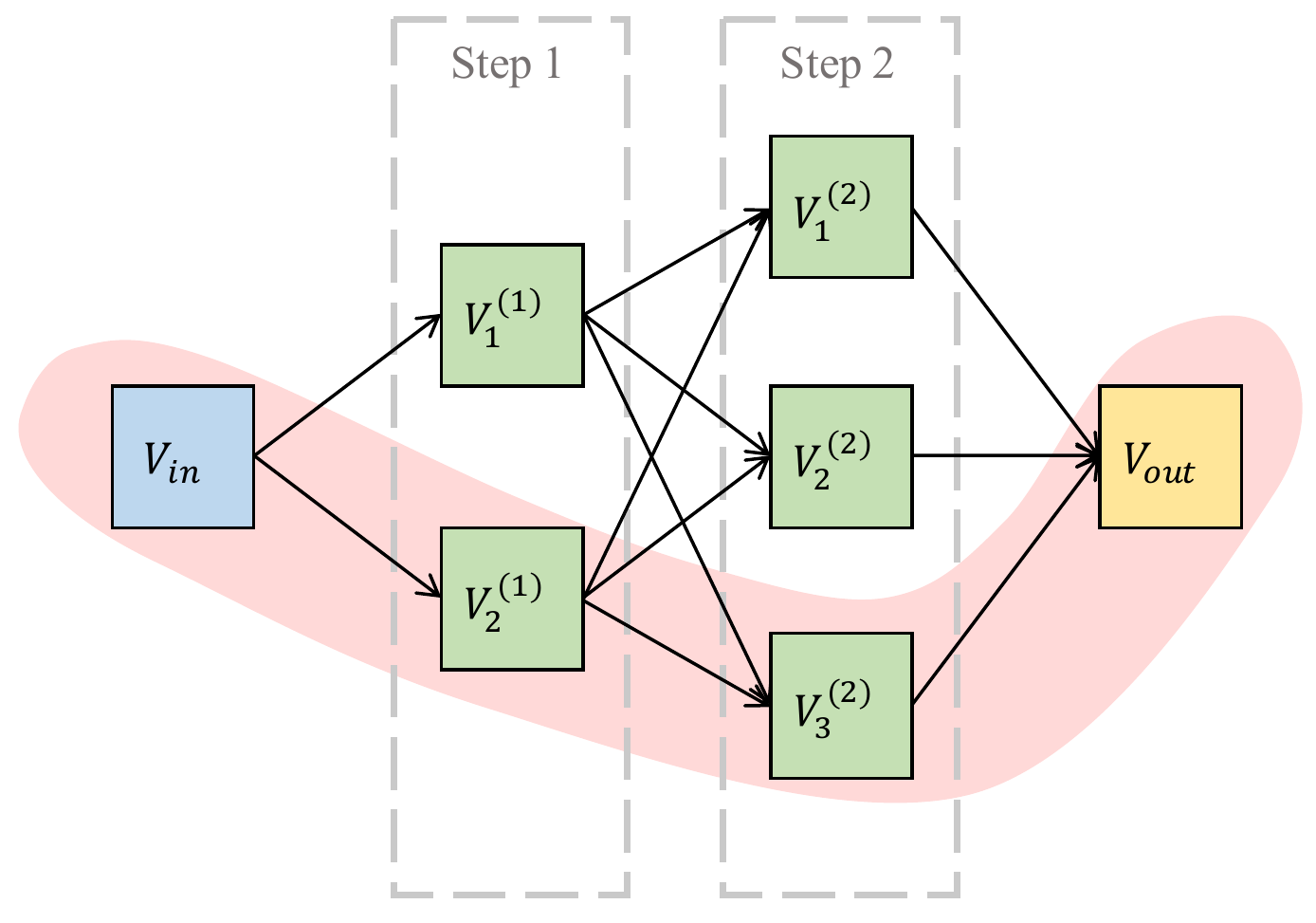}
\caption{A toy example of data analytic pipeline. One possible pipeline path, flowing from the input $V_{in}$ to the output $V_{out}$, is highlighted in shaded area.}
\label{fig:sample_pipeline}
\end{figure}

The problem of tuning data analytic pipelines can be formalized as an optimization problem:

\begin{problem}
\label{def}
Given a data analytic pipeline $G$ with input data $\mathcal{D}$, resource budget $T$, evaluation metric function $m(G,\mathcal{D};\bm{p},\bm{\lambda}_{\bm{p}})$, resource cost of running pipeline $\tau(G,\mathcal{D};\bm{p},\bm{\lambda}_{\bm{p}})$, how to find the pipeline path $\bm{p}$ and its hyperparameters $\bm{\bm{\lambda}}_{\bm{p}}$ with best performance $m^\star$?  
\end{problem}
The performance of the best pipeline path is denoted by $m^{\star} = \min_{\bm{p}, \bm{\lambda}_{\bm{p}}} m(G,\mathcal{D};\bm{p},\bm{\lambda}_{\bm{p}})$ subject to budget $T$. The objective is to approach the optimal performance within the budget $T$ via optimizing over $\bm{p}, \bm{\lambda}_{\bm{p}}$; i.e., we would like our solution $\widehat{\bm{p}}, \widehat{\bm{\lambda}}_{\widehat{\bm{p}}}$ to satisfy 
$m(G,\mathcal{D};\widehat{\bm{p}}, \widehat{\bm{\lambda}}_{\widehat{\bm{p}}}) \leq m^{\star} + \epsilon$ for small values of $\epsilon$.

To efficiently tackle this problem, we propose a two-layer Bayesian optimization approach named Fast LineAr SearcH (FLASH). We generally introduce the idea of linear model and describe the algorithm in Section~\ref{sec:linear}. An immediate advantage of using linear model is that we can use more principled initialization instead of random initialization, as discussed in Section~\ref{sec:optimal}. We use cost-sensitive modeling to prune the pipeline, as described in Section~\ref{sec:costsensitive}. Finally, we accelerate the entire optimization procedure via pipeline caching, which we describe in Section~\ref{sec:caching}.

\begin{table}[t]
\centering
\caption{Mathematical notations used in this paper.}
\label{table:notation}
\vspace{5pt}
\begin{tabular}{ll}
\toprule
Symbol    & Description                  \\ \toprule
$G$       & data analytic pipeline       \\
$V$       & set of algorithms in $G$       \\
$E$       & set of dependency between algorithms \\
$K$       & total number of steps in $G$ \\
$N$       & total number of algorithms in $G$ \\
$\mathcal{D}$	& input data of pipeline \\
$V_{in}$  & input vertex of $G$                 \\
$V_{out}$ & output vertex of $G$                \\
$V_i^{(k)}$ & $i$th algorithm in $k$th step \\
$\bm{p}^{(k)}$ & one-hot vector for $k$th step \\
$\bm{\lambda}_{\bm{p}}^{(k)}$ & hyperparameters for $k$th step \\
$\bm{p}$  & pipeline path \\
$\bm{\lambda}_{\bm{p}}$ & all hyperparameters of $\bm{p}$ \\
$m(\cdot)$ & evaluation metric function \\
$\tau(\cdot)$ & time cost of running pipeline \\
$T_{\text{init}}$ & budget for Phase 1 \\
$T_{\text{prune}}$ & budget for Phase 2 \\
$T_{\text{total}}$ & budget for Phase 3 \\
\bottomrule
\end{tabular}
\end{table}

\subsection{Two-layer Bayesian Optimization}
\label{sec:linear}
Inspired by the performance of linear regression under model misspecification \cite{white1982maximum,flynn2013efficiency,lv2014model} and superior sample complexity compared to more flexible nonparametric techniques \cite{wasserman2006all}, we seek parametric models for propagation of error (or other quantitative metrics) in analytic pipelines.
The high level idea of FLASH is as follows: we propose a linear model for estimating the propagation of error (or any other metric) in a given analytic pipeline. 
The linear model assumes that the performance of algorithms in different steps are independent, and the final performance is additive from all algorithms.  That is, we can imagine that each algorithm is associated with a performance metric, and the total performance of a pipeline path is the sum of the metrics for all algorithms in the path.
This linear model will replace the Gaussian process or random forest in the initial stages of the pipeline tuning process.  In the rest of this section, we provide the details of Bayesian optimization with our linear model.

We apply the linear model only to the pipeline selection vector $\bm{p}$ and assume that the variations due to hyperparameters of the algorithms are captured in the noise term.  That is, we assume that the error of any pipeline path $\bm{p}$ can be written as 
\begin{equation*}
m = \bm{\beta}^{\top}\bm{p} + \varepsilon
\end{equation*}
\noindent where $\bm{\beta}\in \mathbb{R}^{N}$ denotes the parameters of the linear model. 
Given a set of observations of the algorithm selection and the corresponding evaluation metric for the selected pipeline path in the form of $(\bm{p}_i, m_i)$, $i = 1, \ldots, n$, we can fit this model and infer its mean $\mu(\bm{p})$ and variance $\sigma^2(\bm{p})$ of the performance estimation for any new pipeline path represented by $\bm{p}$. 
In particular, let the design matrix $\bm{P} \in \mathbb{R}^{n \times N}$ denote the stacked version of the pipeline paths, i.e.,  $\bm{P} = [\bm{p}_1, \ldots, \bm{p}_n]^{\top}$, and $\bm{m} \in \mathbb{R}^n$ be the corresponding response values of the evaluation metrics, $\bm{m} = [m_1, \ldots, m_n]$. We use the following $L_2$ regularized linear regression to obtain the robust estimate for $\bm{\beta}$ from history observations:
\begin{equation}
\widehat{\bm{\beta}}(\bm{P}, \bm{m}) = \argmin_{\bm{\beta}}\left\{\frac{1}{n}\left\|\bm{P}\bm{\beta} - \bm{m} \right\|_2^2 + \lambda \|\bm{\beta}\|_2^2\right\}
\label{eq:ridge}
\end{equation}
\noindent where for any vector $\bm{x} \in \mathbb{R}^n$ the $L_2$ norm is defined as $\|\bm{x}\|_2 = \sqrt{\left(\sum_{i=1}^nx_i^2 \right)}$.
The predictive distribution for the linear model is Gaussian with mean $\widehat{\mu}_{\bm{p}} = \widehat{\bm{\beta}}^{\top}\bm{p}$ and variance $\widehat{\sigma}_{\bm{p}} = \sigma_{\varepsilon}^2(1+\bm{p}^{\top}(\bm{P}^{\top}\bm{P}+\lambda\bm{I})^{-1}\bm{p})$ where $\sigma_{\varepsilon}^2$ is the variance of noise in the model. 
We estimate $\sigma_{\varepsilon}$ as follows:
the residual in the $i$th observation is computed as $\widehat{\epsilon}_i = \widehat{\mu}_i - m_i$ where $\widehat{\mu}_i = \widehat{\bm{\beta}}^{\top}\bm{p}_i$ is the estimate of $m_i$ by our model. Thus, the variance of the residual can be found as $\widehat{\sigma}_\epsilon^2 = var(\widehat{\mu}_i - m_i)$ where $var(\cdot)$ denotes the variance operator.

\begin{algorithm}[!h]
\caption{Fast Linear Search (FLASH)}
\label{alg:flash}
\DontPrintSemicolon
\SetKwInOut{Input}{input}\SetKwInOut{Output}{output}

\Input{Data analytic pipeline $G$; input data $\mathcal{D}$; total budget for entire optimization $T_{\text{total}}$; budget for initialization $T_{\text{init}}$; budget for pipeline pruning $T_{\text{prune}}$; number of top pipeline paths $r$}
\Output{Optimized pipeline configuration $\widehat{\bm{p}}$ and $\widehat{\bm{\lambda}}_{\bm{\widehat{p}}}$}

\tcc{Phase 1: Initialization (Section~\ref{sec:optimal})}
\While{budget $T_{\text{init}}$ not exhausted}{
	 $\bm{p} \gets$ new pipeline path from Algorithm \ref{alg:initial} \label{line:yield} \;
	$\bm{\lambda}_{\bm{p}} \gets$ random hyperparameters for $\bm{p}$ \;
	$m, \tau \gets$ RunPipeline($G, \mathcal{D}; \bm{p}, \bm{\lambda}_{\bm{p}}$) with Algorithm~\ref{alg:caching}\;
	$\bm{P} \gets [\bm{P}; \bm{p}^{\top}]$, 
		   $\bm{m} \gets [\bm{m}, m]$,
		   $\bm{\tau} \gets [\bm{\tau}, \tau]$ \;
	$\bm{\beta} \gets \widehat{\bm{\beta}}(\bm{P}, \bm{m})$,
		   $\bm{\beta}_\tau \gets \widehat{\bm{\beta}_\tau}(\bm{P}, \bm{\tau})$ using Eq.~(\ref{eq:ridge}) \;
}

\tcc{Phase 2: Pipeline pruning (Section~\ref{sec:costsensitive})}
\While{budget $T_{\text{prune}}$ not exhausted}{
	$\bm{p} \gets \argmax_{\bm{p}}EIPS(\bm{p}, \bm{P}, \bm{\beta}, \bm{\beta}_\tau)$ using Eq.~(\ref{eq:EIPS}) \;
	$\bm{\lambda}_{\bm{p}} \gets$ random hyperparameters for $\bm{p}$ \;
	$m, \tau \gets$ RunPipeline($G, \mathcal{D}; \bm{p}, \bm{\lambda}_{\bm{p}}$) with Algorithm~\ref{alg:caching} \;
	$\bm{P} \gets [\bm{P}; \bm{p}^{\top}]$,
		   $\bm{m} \gets [\bm{m}, m]$,
		   $\bm{\tau} \gets [\bm{\tau}, \tau]$ \;
	$\bm{\beta} \gets \widehat{\bm{\beta}}(\bm{P}, \bm{m})$,
		   $\bm{\beta}_\tau \gets \widehat{\bm{\beta}}_\tau(\bm{P}, \bm{\tau})$ using Eq.~(\ref{eq:ridge}) \; 
}
$G' \gets$ construct subgraph of $G$ with top $r$ pipeline paths with largest $EIPS(\bm{p}, \bm{P}, \bm{\beta}, \bm{\beta}_\tau)$ using Eq.~(\ref{eq:EIPS}) \label{line:prune}\;

\tcc{Phase 3: Pipeline tuning}
$S \gets$ history observations within $G'$ \;
Initialize model $\mathcal{M}$ given $S$ \; 
\While{budget $T_{\text{total}}$ not exhausted}{
	$\bm{p}, \bm{\lambda}_{\bm{p}} \gets$ next candidate from $\mathcal{M}$ \;
	$m \gets$ RunPipeline($G', \mathcal{D}; \bm{p}, \bm{\lambda}_{\bm{p}}$) with Algorithm~\ref{alg:caching} \;
	$S \gets S \cup \{(\bm{p}, \bm{\lambda}_{\bm{p}}, m)\}$ \;
	Update $\mathcal{M}$ given $S$ \;
}

$\widehat{\bm{p}}$, $\widehat{\bm{\lambda}}_{\bm{\widehat{p}}} \gets$ Best configuration so far found for $G'$ \;

\end{algorithm}

To perform Bayesian optimization with linear model, we use the popular \textit{Expected Improvement} (EI) criteria, which recommends to select the next sample $\bm{p}_{t+1}$ such that the following acquisition function is maximized. The acquisition function represents the expected improvement over the best observed result $m^+$ at a new pipeline path $\bm{p}$ \cite{thornton2013auto}:
\begin{align}
\label{eq:EI}
EI(\bm{p}) =  \mathbb{E}[I_{m^+}(\bm{p})]  = \sigma_{\bm{p}} [u\Phi(u) + \phi(u)]
\end{align}
where $\quad u = \frac{m^+ - \xi - \mu_{\bm{p}}}{\sigma_{\bm{p}}}$ and $\xi$ is a parameter to balance the trade-off between exploitation and exploration. EI function is maximized for paths with small values of $m_{\bm{p}}$ and large values of $\sigma_{\bm{p}}$, reflecting the exploitation and exploration trade-offs, respectively.
To be more specific, larger $\xi$ encourages more exploration in selecting the next sample. The functions $\Phi(\cdot)$ and $\phi(\cdot)$ represent CDF and PDF of standard normal distribution, respectively.
The idea of Bayesian optimization with EI is that at each step, we compute the EI with the predictive distribution of the existing linear model and find the pipeline path that maximizes EI. We choose that path and run the pipeline with it to obtain a new $(\bm{p}_i, m_i)$ pair. We use this pair to refit and update our linear model and repeat the process. Later on we also present an enhanced version of EI via normalizing it by cost called \textit{Expected Improvement Per Second} (EIPS).

We provide the full details of FLASH in Algorithm~\ref{alg:flash}. While the main idea of FLASH is performing Bayesian optimization using linear model and EI, it has several additional ideas to make it practical.  Specifically, FLASH has three phases:
\begin{itemize}
\item Phase 1, we initialize the algorithm using ideas from optimal design, see Section~\ref{sec:optimal}. The budget is bounded by $T_{\text{init}}$. 
\item Phase 2, we leverage Bayesian optimization to find the top $r$ best pipeline paths and prune the pipeline $G$ to obtain simpler one $G'$. The budget is bounded by $T_{\text{prune}}$\footnote{In practice, it is better to use the time normalized EI (that is EIPS) during Phase 2; this idea is described in Section~\ref{sec:costsensitive}.}.
\item Phase 3, we use general model-based Bayesian optimization methods to fine-tune the pruned pipeline together with their hyperparameters.  
\end{itemize}
In Phase 3, we use state-of-the-art Bayesian optimization algorithm, either SMAC or TPE. These algorithms are iterative: they use a model $\mathcal{M}$ such as Gaussian process or random forest and use EI to pick up a promising pipeline path with hyperparameters for running, and then update the model with the new observation just obtained, and again pick up the next one for running.
The budget is bounded by $T_{\text{total}}$. 
Note that our algorithm is currently described for a sequential setting but can be easily extended to support parallel runs of multiple pipeline paths as well.

\subsection{Optimal Design for Initialization}
\label{sec:optimal}
Most Bayesian optimization algorithms rely on random initialization which can be inefficient; for example, it may select duplicate pipeline paths for initialization. Intuitively, the pipeline paths used for initialization should cover the pipeline graph well, such that all algorithms are included enough times in the initialization phase. The ultimate goal is to select a set of pipeline paths for initialization such that the error in estimation of $\bm{\beta}$ is minimized. 
Given our proposed linear model, we can find the optimal strategy for initialization to make sure the pipeline graph is well covered and the tuning process is robust. In this section, we describe different optimality criteria studied in statistical experiment design \cite{atkinson2012optimum,flaherty2005robust} and active learning  \cite{Settles2012}, and design an algorithm for initialization step of FLASH. 

\begin{algorithm}[t]
\caption{Initialization with Optimal Design}
\label{alg:initial}
\DontPrintSemicolon
\SetKwInOut{Input}{input}\SetKwInOut{Output}{output}

\tcc{\underline{Batch version}}
\Input{$B$ initial candidates $\{\bm{p}_i\}_{i=1}^{B}$; number of desired pipeline paths $n_{\text{init}}$  }
\Output{ Optimal set of pipeline paths $Q$ }

$\bm{p}_1 \gets$ random pipeline path for initialization \;
$\bm{H} \gets \bm{p}_1\bm{p}_1^{\top}$ \;
$Q \gets \{\bm{p}_1\}$ \;
\For{$\ell = 2, \ldots, n_{\text{init}}$}{
    $j^{\star} \gets \argmax_{j}D_\ell(\bm{H} + \bm{p}_j\bm{p}_j^{\top})$ for $j=1, \ldots, B$. \;
    $\bm{H} \gets \bm{H} + \bm{p}_{j^{\star}}\bm{p}_{j^{\star}}^{\top}$ \;
    $Q \gets Q\cup\{\bm{p}_{j^{\star}}\}$ \;
}

\nonl \hrulefill

\tcc{\underline{Online version}}
\Input{$B$ candidates $\{\bm{p}_i\}_{i=1}^{B}$; current Gram matrix $\bm{H}$}
\Output{Next pipeline path $\bm{p}_{j^{\star}}$, $j^{\star} \in \{1, \ldots, B\}$}

$j^{\star} \gets \argmax_{j}D_\ell(\bm{H} + \bm{p}_j\bm{p}_j^{\top})$ for $j=1, \ldots, B$ \;

\end{algorithm}

Given a set of pipeline paths with size $n$, there are several different optimality criteria in terms of the eigenvalues of the Gram matrix $\bm{H} = \sum_{i=1}^{n}\bm{p}_i\bm{p}_i^{\top}$ as follows \cite{Settles2012}:
\begin{description}
\item[A-optimality:] maximize $\sum_{\ell=1}^{n}\lambda_{\ell}(\bm{H})$.
\item[D-optimality:] maximize $\prod_{\ell=1}^{n}\lambda_{\ell}(\bm{H})$.
\item[E-optimality:] maximize $\lambda_{n}(\bm{H})$, the $n$th largest eigenvalue.
\end{description}
It is easy to see that any arbitrary set of pipeline path designs satisfies the A-optimality criterion. 
\begin{prop}
Any arbitrary set of pipeline paths with size $n$ is a size-$n$ A-optimal design.
\label{prop:Aoptimal}
\end{prop}
\begin{proof}
For any arbitrary set of pipeline paths with size $n$, we have:
\begin{align*}
\sum_{\ell=1}^{n}\lambda_{\ell}(\bm{H})& = \mathrm{tr}(\bm{H}) = \mathrm{tr}\left(\sum_{i=1}^{n}\bm{p}_i\bm{p}_i^{\top} \right) = \sum_{i=1}^{n}\mathrm{tr}\left(\bm{p}_i\bm{p}_i^{\top} \right) = nK.
\end{align*}
The last step is due to particular pattern of $\bm{p}$ in our problem.  Thus, we show that $\sum_{\ell=1}^{n}\lambda_{\ell}(\bm{H})$ is constant, independent of the design of pipeline paths. 
\end{proof}

Proposition \ref{prop:Aoptimal} rules out use of A-optimality in pipeline initialization. 
Given the computational complexity of \mbox{E-optimality} and the fact that it intends for optimality in the extreme cases, we choose \mbox{D-optimality} criterion. The \mbox{D-optimality} criterion for design of optimal linear regression can be stated as follows: suppose we are allowed to evaluate $n_{\mathrm{init}}$ samples $\bm{p}_i$, $i=1, \ldots, n_{\mathrm{init}}$, these samples should be designed such that the determinant of the Gram matrix $\bm{H}$ is maximized. 
While we can formulate an optimization problem that directly finds $\bm{p}_i$ values, we found that an alternative approach can be computationally more efficient.  In this approach, we first generate $B$ candidate pipeline paths for an integer $B$ larger than the number of algorithms in the pipeline $N$. This set may include all possible pipeline paths if the total number of paths is small. Then, our goal becomes selecting a subset of size $n_{\mathrm{init}}$ from them. 
We can formulate the optimal design as follows
\begin{align}
\bm{a}^{\star}& = \argmax_{\bm{a}}~ \det\left\{ \sum_{i=1}^{B}a_i\bm{p}_i\bm{p}_i^{\top} \right\}
\label{eq:optimalDesign}\\
s.t. &\qquad \bm{a}\in \{0, 1\}^{K}, \qquad \bm{1}^{\top}\bm{a} = n_{\mathrm{init}}.\nonumber
\end{align}
The last constraint $\bm{1}^{\top}\bm{a} = n_{\mathrm{init}}$ indicates that only $n_{\mathrm{init}}$ pipeline paths should be selected.
The objective function is concave in terms of continuous valued $\bm{a}$ \cite[Chapter 3.1.5]{boyd2004convex}. Thus, a traditional approach is to solve it by convex programming after relaxation of the integrality constraint on $\bm{a}$. The matrix in the argument of the determinant is only $N$-dimensional which means calculation of the determinant should be fast. Nestrov's accelerated gradient descent \cite{nesterov2007gradient} or Frank-Wolfe's \cite{jaggi2013revisiting} algorithms can be used for efficiently solving such problems. 

An even faster solution can be found by using greedy forward selection ideas which are fast and popular for optimal experiment design, for example see \cite{robertazzi1989accelerated,he2010laplacian,krause2011submodularity} and the references therein. To apply greedy technique to our problem, we initialize the solution by picking one of the pipeline path $\bm{H} = \bm{p}_i\bm{p}_i^{\top}$. Then, at $\ell$th step, we add the path that maximizes $j^{\star} = \argmax_{j}D_\ell(\bm{H} + \bm{p}_j\bm{p}_j^{\top})$ where $D_\ell(\bm{H}) = \prod_{i=1}^{\min(\ell,p)}\lambda_i(\bm{H})$ denotes the product of top $\min(\ell,p)$ eigenvalues of its argument. The algorithm is described in Algorithm~\ref{alg:initial}. The optimization problem in Eq.~(\ref{eq:optimalDesign}) appears in other fields such as optimal facility location and sensor planning where greedy algorithm is known to have a $1-\frac{1}{e}$ approximation guarantee \cite{calinescu2011maximizing,shamaiah2010greedy}.

One further desirable property of the greedy algorithm is that it is easy to run it under a time budget constraint. We call this version the online version in Algorithm \ref{alg:initial}, where instead of a fixed number of iteration $n_{\mathrm{init}}$, we run it until the exhaustion of our time budget. See Line \ref{line:yield} in Algorithm~\ref{alg:flash} and the online version of Algorithm~\ref{alg:initial}.

\subsection{Cost-sensitive Modeling}
\label{sec:costsensitive}
The \textit{Expected Improvement} aims at approaching the true optimal (doing well) within a small number of function evaluations (doing fast). However, the time cost of each function evaluation may vary a lot due to different settings of hyperparameters. This problem is particularly highlighted in pipeline configurations tuning, since the choice of algorithms can make a huge difference in running time. Therefore, fewer pipeline runs are not always ``faster'' in terms of wall-clock time. Also, in practice, what we care about is the performance we can get within limited resource budget, rather than within certain evaluation times. That is why we need cost-sensitive modeling  for the pipeline tuning problem.

\emph{Expected Improvement Per Second} (EIPS) \cite{snoek2012practical} proposes another acquisition function for tuning of a single learning algorithm by dividing the EI of each hyperparameter by its runtime. 
To apply EIPS in pipeline tuning problem, we use a separate linear model to model the total runtime of pipeline paths. Similar to the linear model for error propagation, the linear model for time assumes that on a pipeline path each algorithm partly contributes to the total time cost and the runtimes are additive. To apply the linear model, we replace the performance metric $\bm{m}$ with the cost metric $\bm{\tau}$. The linear cost model parametrized by $\bm{\beta_{\tau}}$ can be efficiently updated using Eq.~(\ref{eq:ridge}). As described in Algorithm~\ref{alg:flash}, $\bm{\beta_{\tau}}$ will be updated together with $\bm{\beta}$ at the end of Phase 2. We note that, in practice, the budget $T$ and the cost $\tau(\cdot)$ can be any quantitative costs of budgeted resources (e.g., money, CPU time), which is a natural generalization of our idea.

With the cost model above, we get the cost-sensitive acquisition function over the best observed result $m^+$ at a new pipeline path $\bm{p}$:
\begin{align}
\label{eq:EIPS}
& EIPS(\bm{p}, \bm{P}, \bm{\beta}, \bm{\beta}_\tau) 
= \frac{\mathbb{E}[I_{m^+}(\bm{p})]}{\mathbb{E}[\log{\tau(\bm{p})}]} 
= \frac{\sigma_{\bm{p}} [u\Phi(u) + \phi(u)]}{\mathbb{E}[\log{\tau(\bm{p})}]} \\
& \text{where} \quad u = \frac{m^+ - \xi - \mu_{\bm{p}}}{\sigma_{\bm{p}}}. \nonumber
\end{align}
Here the dependency in $\bm{\beta}$ and $\bm{\beta}_\tau$ is captured during computation of $\mu_{\bm{p}}$, $\sigma_{\bm{p}}$, and $\tau(\bm{p})$. 
We take logarithm of cost $\tau(\cdot)$ to compensate the large variations in the runtime of different algorithms.
This acquisition function balances ``doing well'' and ``doing fast'' in selecting the next candidate path to run. During the optimization, it will help avoid those costly paths with poor expected improvement. More importantly, at the end of Phase 2 in Algorithm~\ref{alg:flash}, EIPS is responsible to determine the most promising paths, which perform better but cost less, to construct a subgraph for the last phase fine-tuning. For this purpose, we set the exploration parameter $\xi$ to 0 to only select (Line \ref{line:prune} in Algorithm \ref{alg:flash}).

\subsection{Pipeline Caching}
\label{sec:caching}
\begin{algorithm}[t]
\caption{Pipeline Caching}
\label{alg:caching}
\DontPrintSemicolon
\SetKwInOut{Input}{input}\SetKwInOut{Output}{output}

\Input{Data analytic pipeline $G$; input data $\mathcal{D}$; pipeline configuration $\bm{p}$ and $\bm{\lambda}_{\bm{p}}$ to be run; available caching budget $T_{cache}$; current cache pool $C$}

$\mathcal{D}^{(1)} \gets \mathcal{D}$ \;
\tcc{Run pipeline with cache pool}
\For {$k \gets 1, \ldots, K$}{
	$h \gets$ Hash($\bm{p}^{(1)} \ldots \bm{p}^{(k)}, \bm{\lambda}_{\bm{p}}^{(1)} \ldots \bm{\lambda}_{\bm{p}}^{(k)}$) \;
	\If {$h \in C$} {
		$\mathcal{D}^{(k+1)} \gets$ cached result from $C$\;
    }
	\Else { 
		$\mathcal{D}^{(k+1)} \gets$ RunAlgorithm($G, \mathcal{D}^{(k)}, \bm{p}^{(k)}, \bm{\lambda}_{\bm{p}}^{(k)}$) \;
		$C \gets C \cup \{\langle h, \mathcal{D}^{(k+1)} \rangle \}$ \;
	}
}
\tcc{Clean up cache pool when necessary}
\If {$T_{cache}$ exhausted} {
	Discard least recently used (LRU) items in $C$ \;
}
\end{algorithm}

During the experiments, we note that many pipeline runs have overlapped algorithms in their paths. Sometimes these algorithms have exactly the same pipeline path and the same hyperparameter settings along the path. This means that we are wasting time on generating the same intermediate output again and again. For example, consider the min-max normalization algorithm in the first pipeline step: this algorithm will be executed many times, especially when it performs well so that Bayesian optimization methods prefer to choose it.

To reduce this overhead, we propose a pipeline caching algorithm, as described in Algorithm~\ref{alg:caching}. When running a pipeline, we check the cache before we run each algorithm. If it turns out to be a cache hit, the result will be immediately returned from cache. Otherwise, we run the algorithm and cache the result. There is a caching pool (e.g., disk space, memory usage) for this algorithm. We use the Least Recently Used (LRU) strategy to clean up the caching pool when budget becomes exhausted.

Caching can significantly reduce the cost of pipeline runs, and accelerates all three phases of FLASH. Algorithms closer to the pipeline input vertex, usually the data preprocessing steps, have higher chance to hit the cache. In fact, when we deal with large datasets on real-world problems, the preprocessing step can be quite time-consuming such that caching can be very efficient.

\section{Benchmark Experiments}
\label{sec:exp}
In this section, we perform extensive experiments on a number of benchmark datasets to evaluate our algorithm compared to the existing approaches. Then we study the impact of different algorithm choices in each component of FLASH. 

\paragraph{Benchmark Datasets}

We conduct experiments on a group of public benchmark datasets on classification task, including \textsc{Madelon}, MNIST, MRBI and \textsc{Convex}\footnote{The benchmark datasets are publicly available at \url{http://www.cs.ubc.ca/labs/beta/Projects/autoweka/datasets}.}. These prominent datasets have been widely used to evaluate the effectiveness of Bayesian optimization methods \cite{thornton2013auto, feurer2015efficient, bergstra2011algorithms}. We follow the original train/test split of all the datasets. Test data will never be used during the optimization: the once and only usage of test data is for offline evaluations to determine the performance of optimized pipelines on unseen test set. In all benchmark experiments, we use percent error rate as the evaluation metric.

\paragraph{Baseline Methods}
As discussed in Section \ref{sec:back}, SMAC \cite{hutter2011sequential} and TPE \cite{bergstra2011algorithms} are the state-of-the-art algorithms for Bayesian optimization \cite{thornton2013auto,feurer2015efficient,komer2014hyperoptsklearn}, which are used as baselines. 
Note that Spearmint \cite{snoek2012practical}, a Bayesian optimization algorithm based on Gaussian process is not applicable since it does not provide a mechanism to handle the hierarchical space \cite{eggensperger2013towards}.
Besides SMAC and TPE, we also choose random search as a simple baseline for sanity check. Thus, we compare  both versions of our method FLASH (with SMAC in Phase 3) and FLASH$^\star$ (with TPE in Phase 3) against three baselines in the experiments.

\noindent{\bf Implementation:} To avoid possible mistakes in implementing other methods, we choose a general platform for hyperparameter optimization called HPOlib \cite{eggensperger2013towards}, which provides the original implementations of SMAC, TPE, and random search. In order to fairly compare our method with others, we also implement our algorithm on top of HPOlib, and evaluate all the compared methods on this platform. We make the source code of FLASH publicly available at \url{https://github.com/yuyuz/FLASH}.

\paragraph{Experimental Settings}

We build a general data analytic pipeline based on scikit-learn \cite{pedregosa2011scikit}, a popular used machine learning toolbox in Python. We follow the pipeline design of auto-sklearn \cite{feurer2015efficient}. There are four computational steps in our pipeline: 1) feature rescaling, 2) sample balancing, 3) feature preprocessing, and 4) classification model. Each step has various algorithms, and each algorithm has its own hyperparameters. Adjacent steps are fully connected. In total, our data analytic pipeline contains 33 algorithms distributed in four steps, creating 1,456 possible pipeline paths with 102 hyperparameters (30 categorical and 72 continuous), which creates complex high-dimensional and highly conditional search space. Details and statistics of this pipeline are listed in Table \ref{table:pipeline_setting} in Appendix \ref{sec:appendix}.

In all experiments, we set a wall-clock time limit of 10 hours for the entire optimization, 15 minutes time limit and 10GB RAM limit for each pipeline run. We perform 10 independent optimization runs with each baseline on each benchmark dataset. All experiments were run on Linux machines with Intel Xeon E5-2630 v3 eight-core processors at 2.40GHz with 256GB RAM. Since we ran experiments in parallel, to prevent potential competence in CPU resource, we use the \texttt{numactl} utility to bound each independent run in single CPU core.

For our algorithm FLASH, we set both $T_{init}$ and $T_{prune}$ as $30$ iterations (equal to the number of algorithms in the pipeline), which can be naturally generalized to other budgeted resources such as wall-clock time or money. We set $\xi$ to $100$ in the EIPS function. Note that the performance are not sensitive to the choices of those parameters. Finally, we set the number of pipeline paths $r$ to $10$ , which works well in generating a reasonable-size pruned pipeline $G'$.
In benchmark experiments, we compare the performance of FLASH without caching to other methods because the pipelines do not have complex data preprocessing data like many real-world datasets have. We will use caching for real-world experiments later in Section \ref{sec:real}. 

\subsection{Results and Discussions}

Table~\ref{table:benchmark_perf} reports the experimental results on benchmark datasets. For each dataset, we report the performance achieved within three different time budgets. As shown in the table, our methods FLASH and FLASH$^\star$ perform significantly better than other baselines consistently in all settings, in terms of both lower error rate and faster convergence. 
For example, on the \textsc{Madelon} dataset, our methods reach around 12\% test error in only 3 hours, while other baselines are still far from that even after 10 hours.

Performing statistical significance test via bootstrapping, we find that often FLASH and FLASH$^\star$ tie with each other on these benchmark datasets. For all the methods, the test error is quite consistent with the validation error, showing that the potential overfitting problem is well prevented by using cross validation. 

Figure~\ref{fig:exp_MRBI_test} plots the convergence curves of median test error rate along with time for all baseline methods. As shown in the figure, after running about 4 hours, FLASH and FLASH$^\star$ start to lead  others with steep drop of error rate, and then quickly converge on a superior performance.

\setlength\tabcolsep{3.5pt}
\begin{table*}[!ht]
\centering
\caption{Performance on both 3-fold cross-validation and test data of benchmark datasets. For each method, we perform 10 independent runs of 10 hours each. Results are reported as the median percent error across the 10 runs within different time budgets. Test data is never seen by any optimization method, which is only used for offline evaluations to compute test error rates. Boldface indicates the best result within a block of comparable methods. We underline those results not statistically significantly different from the best according to a 10,000 times bootstrap test with $p=0.05$.}
\label{table:benchmark_perf}
\vspace{5pt}
\footnotesize
\begin{tabular}{lccccccccccccc}
\toprule
\multirow{2}{*}{\textbf{Dataset}} & \multicolumn{1}{l}{\multirow{2}{*}{\textbf{\begin{tabular}[c]{@{}l@{}}Budget\\ (hours)\end{tabular}}}} & \multicolumn{1}{l}{\textbf{}} & \multicolumn{5}{c}{\textbf{Cross-validation Performance (\%)}} & \multicolumn{1}{l}{\textbf{}} & \multicolumn{5}{c}{\textbf{Test Performance (\%)}} \\ \cmidrule(l){3-14} 
 & \multicolumn{1}{l}{} & \multicolumn{1}{l}{} & \begin{tabular}[c]{@{}c@{}}\textsc{Rand.}\\ \textsc{Search}\end{tabular} & TPE & SMAC & FLASH & FLASH$^\star$ &  & \begin{tabular}[c]{@{}c@{}}\textsc{Rand.}\\ \textsc{Search}\end{tabular} & TPE & SMAC & FLASH & FLASH$^\star$ \\ \midrule
\multirow{3}{*}{\textsc{Madelon}} & 3 &  & 25.16 & 18.90 & 20.25 & {\ul 14.84} & \textbf{14.04} &  & 19.17 & 16.15 & 16.03 & {\ul 12.18} & \textbf{11.73} \\
 & 5 &  & 23.60 & 18.82 & 19.12 & {\ul 14.31} & \textbf{14.04} &  & 18.21 & 15.26 & 15.38 & {\ul 12.18} & \textbf{11.60} \\
 & 10 &  & 20.77 & 17.28 & 17.34 & {\ul 13.87} & \textbf{13.76} &  & 15.58 & 14.49 & 13.97 & {\ul 11.49} & \textbf{11.47} \\ \midrule
\multirow{3}{*}{MNIST} & 3 &  & 7.68 & 6.78 & 6.05 & \textbf{4.93} & {\ul 5.05} &  & 7.75 & 5.41 & 6.11 & \textbf{4.62} & {\ul 4.84} \\
 & 5 &  & 6.58 & 5.94 & 5.83 & \textbf{4.26} & {\ul 4.87} &  & 7.10 & 5.41 & 5.40 & \textbf{3.94} & {\ul 4.57} \\
 & 10 &  & 6.58 & 5.39 & 5.64 & \textbf{4.03} & {\ul 4.46} &  & 6.64 & 5.03 & 5.23 & \textbf{3.78} & {\ul 4.37} \\ \midrule
\multirow{3}{*}{MRBI} & 3 &  & 61.80 & 59.83 & 62.89 & {\ul 57.43} & \textbf{57.08} &  & 60.58 & 59.83 & 60.58 & {\ul 54.72} & \textbf{54.28} \\
 & 5 &  & 58.67 & 58.61 & 58.14 & \textbf{45.11} & 54.25 &  & 56.42 & 58.61 & 55.81 & \textbf{43.19} & 51.65 \\
 & 10 &  & 57.20 & 53.92 & 54.60 & \textbf{41.15} & {\ul 41.90} &  & 54.43 & 52.01 & 52.30 & \textbf{39.13} & {\ul 39.89} \\ \midrule
\multirow{3}{*}{\textsc{Convex}} & 3 &  & 28.14 & 24.70 & 24.69 & \textbf{22.63} & {\ul 23.31} &  & 25.04 & 21.42 & 21.97 & \textbf{20.65} & {\ul 21.04} \\
 & 5 &  & 25.25 & 23.61 & 23.30 & \textbf{21.34} & {\ul 22.02} &  & 23.18 & 21.37 & 20.82 & \textbf{19.56} & {\ul 19.71} \\
 & 10 & \multicolumn{1}{l}{} & 24.51 & 22.21 & 23.30 & \textbf{20.49} & {\ul 20.62} &  & 22.18 & 20.31 & 20.82 & \textbf{18.94} & {\ul 19.01} \\ \bottomrule
\end{tabular}
\end{table*}
\setlength\tabcolsep{6pt}

\begin{figure}[!ht]
\centering
\includegraphics[width=0.6\textwidth]{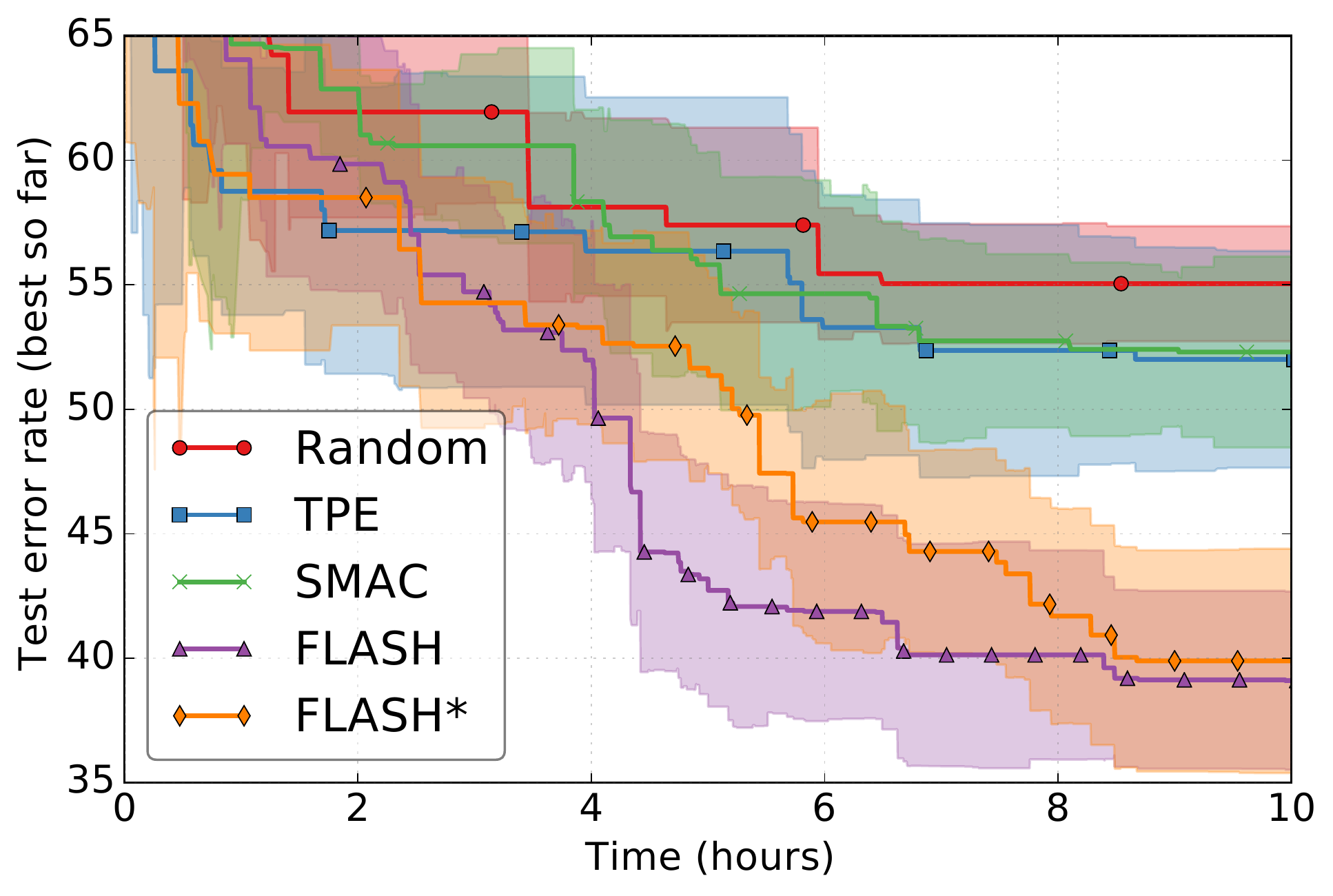}
\caption{Performance of our methods (FLASH and FLASH$^\star$) and other compared methods on MRBI dataset. We show the median percent error rate on test set along with standard error bars (generated by 10 independent runs) over time.}
\label{fig:exp_MRBI_test}
\end{figure}

\begin{figure*}[t]
\centering
\subfigure[The impact of optimal design \newline on MRBI dataset]{
\label{fig:exp_bench_optd}
\includegraphics[width=0.32\textwidth]{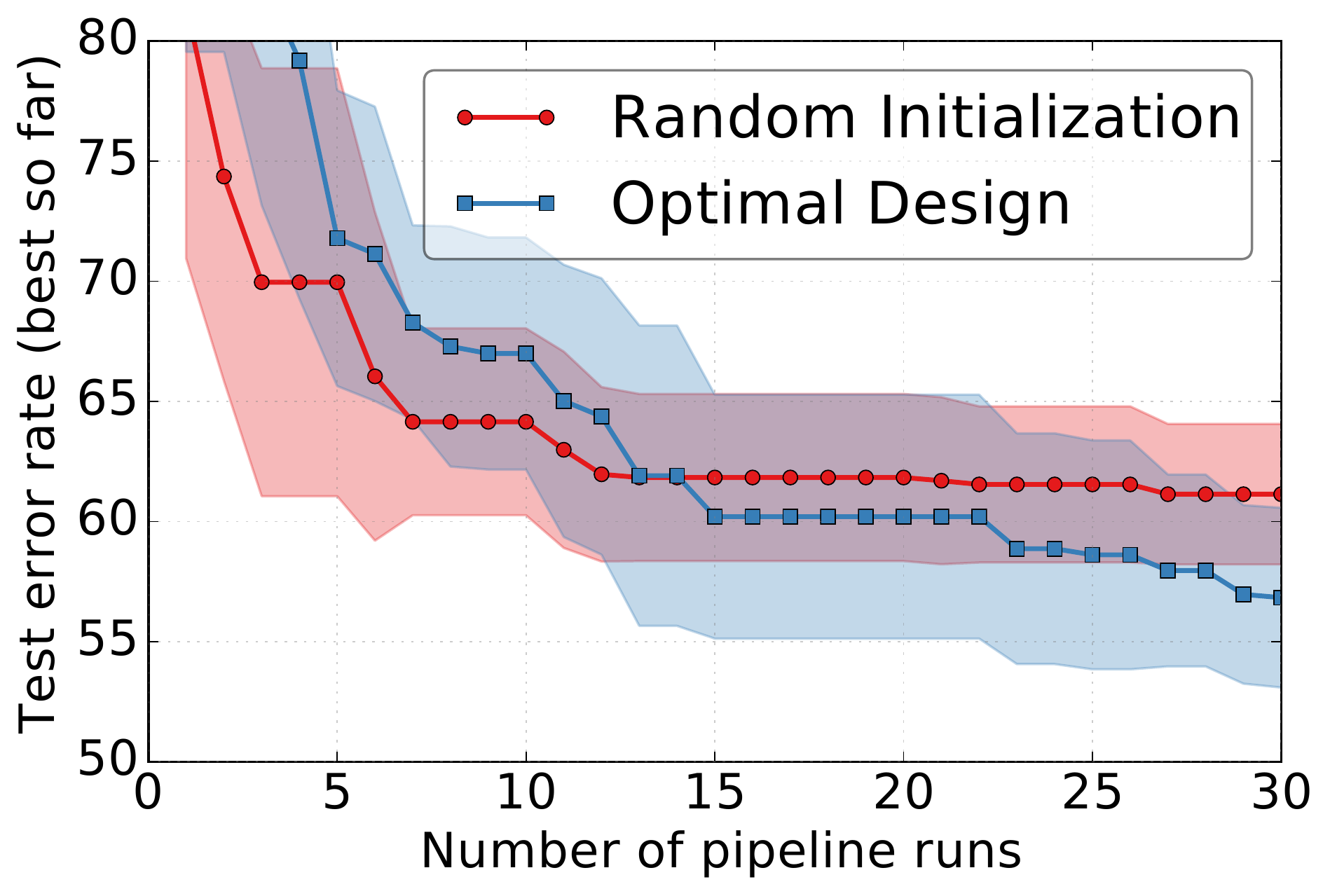}}
\subfigure[The impact of pipeline pruning \newline on MRBI dataset]{
\label{fig:exp_bench_prune}
\includegraphics[width=0.32\textwidth]{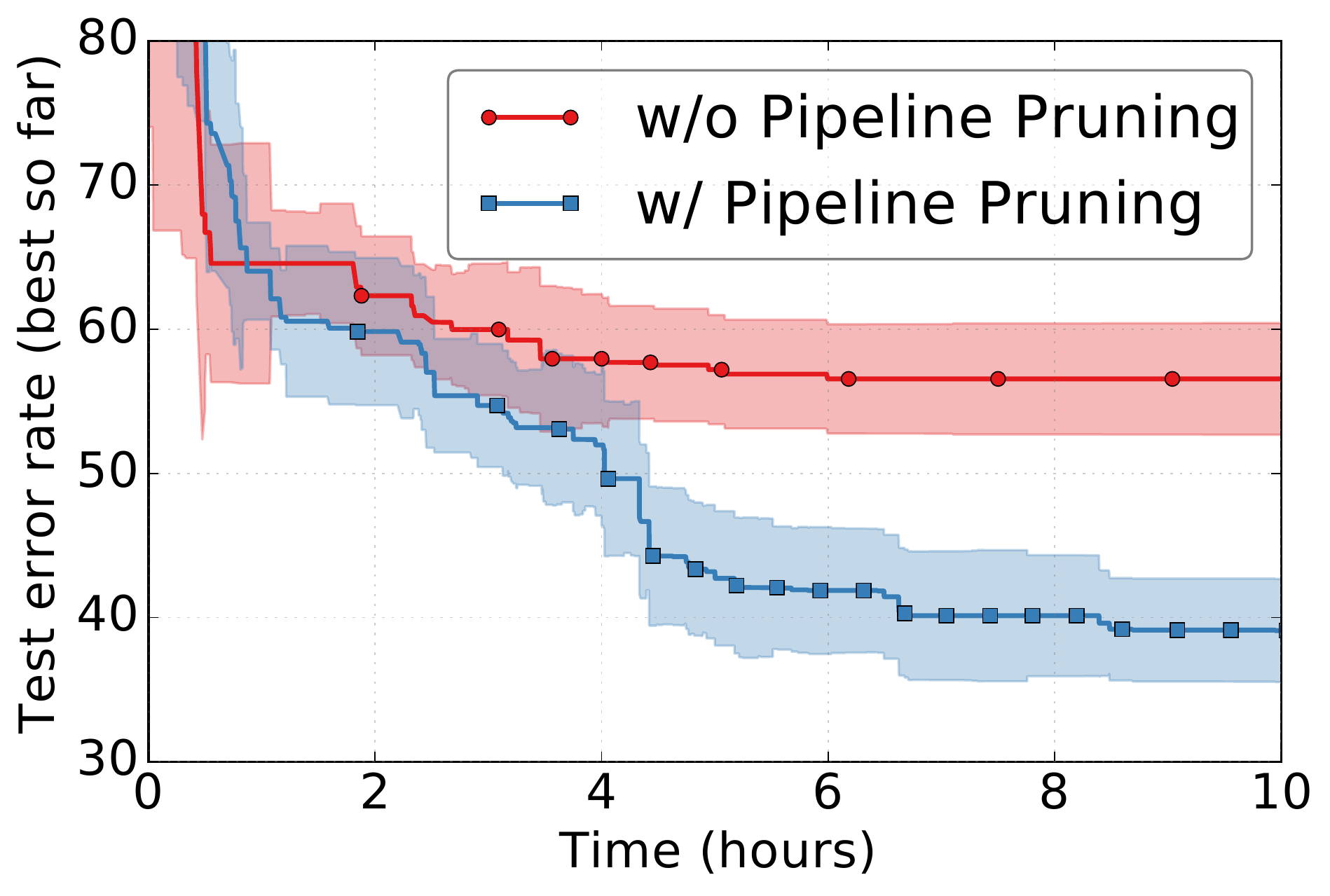}}
\subfigure[The impact of pipeline caching \newline on real-world dataset]{
\label{fig:cache}
\includegraphics[width=0.32\textwidth]{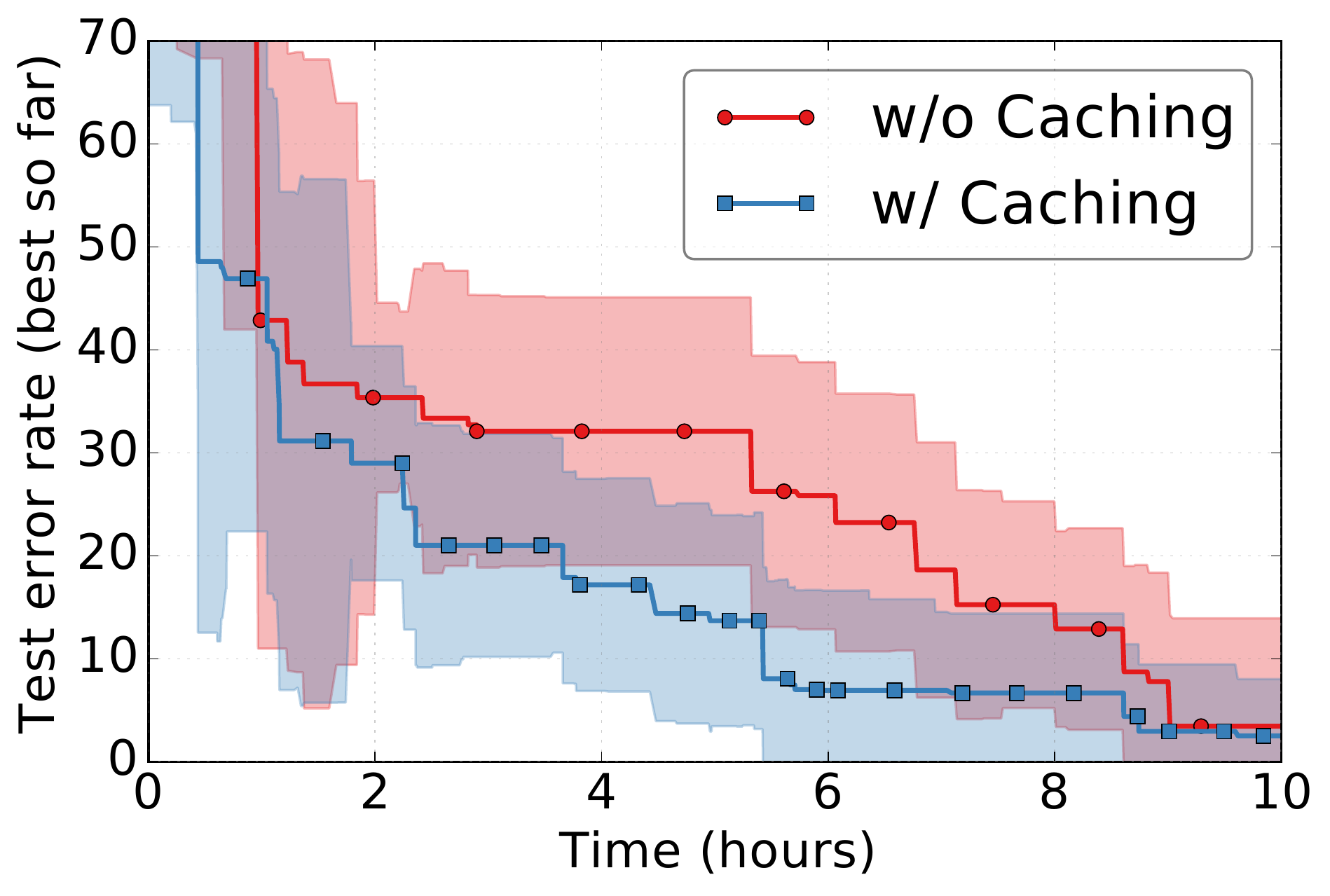}}
\caption{Component analysis experiments. \subref{fig:exp_bench_optd} Optimal design makes the initialization phase more robust. \subref{fig:exp_bench_prune} Pipeline pruning in the second phase of FLASH is the key to its superior performance. \subref{fig:cache} Performance of FLASH without caching and the original FLASH with caching on real-world dataset. In all figures, we show the median error rate on test set along with standard error bars (generated by 10 independent runs). Note that (a) and (b) are plotted with different x-axes; (c) is on a different dataset as (a) and (b).}
\label{fig:exp_bench}
\end{figure*}

\subsection{Detailed Study of FLASH Components}
FLASH has three main components: optimal design for initialization, cost-sensitive model for pipeline pruning, and pipeline caching. To study their individual contributions to the performance gain, we drop out each of the component and compare the performance with original FLASH. Since caching will be used for real-world experiments on large dataset, we describe the analysis of caching component in Section~\ref{sec:real}.  Here we use MRBI dataset for these experiments.

Figure~\ref{fig:exp_bench_optd} shows the difference between using random initialization and optimal design by plotting the performance on initial 30 pipeline runs. The desirable property of optimal design ensures to run reasonable pipeline paths, giving FLASH a head start at the beginning of optimization. While random initialization is not robust enough, especially when the number of pipeline runs is very limited and some algorithms will have no chance to run due to the randomness.
Figure~\ref{fig:exp_bench_prune} shows the impact of pipeline pruning in the second phase of FLASH. Dropping out the pruning phase with EIPS, and using SMAC immediately after Phase 1, we see a major degradation of the performance.
The figure clearly shows that in Phase 2 of FLASH, the linear model with EIPS acquisition function is able to efficiently shrink the search space significantly such that SMAC can focus on those algorithms which perform well with little cost.
This figure confirms the main idea of this paper that a simple linear model can be more effective in searching high-dimensional and highly conditional hyperparameter space.

\section{Real-world Experiments}
\label{sec:real}
In this section, to demonstrate a real-world use case, we apply FLASH on a large de-identified medical dataset for classifying drug non-responders.
We show how our method can quickly find good classifier for differentiating non-responders vs. responders.

\paragraph{Experimental Settings}
With a collaboration with a pharmaceutical company, we created  a balanced cohort of 46,455 patients from a large claim dataset. Patients who have at least 4 times of treatment failure are regarded as drug non-responders (case group). Other patients are responders of the drug (control group). The prediction target is whether a patient belongs to case group or control group. Each patient is associated with a sequence of events, where each event is a tuple of format \texttt{(patient-id, event-id, timestamp, value)}.
Table~\ref{tbl:ucb-data} summarizes the statistics of this clinical dataset, including the count of patient, event, medication, and medication class.

\begin{table}[!t]
\caption{Statistics of the medical dataset.}
\vspace{5pt}
\label{tbl:ucb-data}
\centering
\begin{tabular}{lrrrrr}
\toprule
 & \textbf{\#Patient} & \textbf{\#Event} & \textbf{\#Med} & \textbf{\#Class}\\ \midrule
\sc{train case}    & 18,581 & 982,025   & 434,171 & 547,854   \\
\sc{train control} & 18,582 & 622,777   & 286,198 & 336,579   \\
\sc{test case}     & 4,646  & 245,776   & 108,702 & 137,074   \\
\sc{test control}  & 4,646  & 153,303   & 70,395  & 82,908    \\
\sc{total}         & 46,455 & 2,003,881 & 899,466 & 1,104,415 \\ \bottomrule
\end{tabular}
\end{table}
\setlength\tabcolsep{6pt}

\begin{table}[!t]
\centering
\caption{Performance of real-world dataset. Results are reported using the same settings as Table~\ref{table:benchmark_perf}.}
\vspace{5pt}
\label{table:real_perf}
\begin{tabular}{ccccccc}
\toprule
\multicolumn{1}{l}{\multirow{2}{*}{\textbf{\begin{tabular}[c]{@{}l@{}}Budget\\ (hours)\end{tabular}}}} & \multicolumn{1}{l}{\textbf{}} & \multicolumn{5}{c}{\textbf{Test Performance (\%)}} \\ \cmidrule(l){2-7} 
\multicolumn{1}{l}{} &  & \begin{tabular}[c]{@{}c@{}}\textsc{Rand.}\\ \textsc{Search}\end{tabular} & TPE & SMAC & FLASH & FLASH$^\star$ \\ \midrule
3 &  & 30.32 & 27.03 & 35.40 & \textbf{21.02} & {\ul 23.28} \\
5 &  & 16.66 & 19.09 & 33.22 & \textbf{14.40} & 19.86 \\
10 &  & 11.75 & {\ul 4.86} & 21.03 & \textbf{2.51} & {\ul 3.44} \\ \bottomrule
\end{tabular}
\end{table}

Unlike benchmark experiments, the input to the real-world pipeline is not directly as feature vectors. Given a cohort of patients with their event sequences, like \cite{robert2015athsma} the pipeline for non-responder classification has two more additional steps than the pipeline described in previous benchmark experiments: 
1) \emph{Feature construction} to convert patient event sequence data into numerical feature vectors. This step can be quite time-consuming as advanced feature construction techniques like sequential mining \cite{kunal2015sequential} and tensor factorization \cite{ho2014marble} can be expensive to compute. On this medical dataset, we consider two kinds of parameters in this step: i) frequency threshold to remove rare events (frequency ranging from 2 to 5) ii) various aggregation functions (including binary, count, sum and average) to aggregate multiple occurrence of events into features. The output of this step will be a feature matrix and corresponding classification targets;
2) \emph{Densify} the feature matrix from above feature construction step if necessary. Features of this real-world dataset can be quite sparse. We by default use sparse representation to save space and accelerate computation in some algorithms. Unfortunately, not all algorithm implementations in scikit-learn accept sparse features. A decision has to be made here: either sparse matrix for faster result or dense matrix for broader algorithm choices in later steps.

We run the experiments on same machine, use same parameter setting and same budget as benchmark experiments. We compare our method with the same baselines as benchmark experiments and we continue using error rate as metric. 
Our algorithm has built-in caching mechanism and we will use that. For this real-world dataset, we first compare with baselines with cache enabled. Then we analyze the contribution of caching.

\subsection{Results and Discussions}

Table~\ref{table:real_perf} shows the performance of  our methods compared to baselines when caching is enabled. Due to lack of space we only report the test performance. All cases FLASH and FLASH$^\star$ significantly outperform all the baselines.

Figure~\ref{fig:cache} shows the performance of FLASH without caching and original FLASH with caching on the real-world  medical dataset. With caching, more pipeline paths can be evaluated within given period of time and our EIPS-based path selection leverages caching to select paths with high performance that run fast. As a result, we can see FLASH with caching converges much faster. For example, with caching we can get low test error within 6 hours.

\section{Conclusions}
\label{sec:future}
In this work, we propose a two-layer Bayesian optimization algorithm named FLASH, which enables highly efficient optimization of complex data analytic pipelines. We showed that all  components of FLASH complement each other: 1) our optimal design strategy ensures better initialization, giving a head start to the optimization procedure; 2) the cost-sensitive model takes advantage of this head start, and significantly improves the performance by pruning inefficient pipeline paths; 3) the pipeline caching reduces the cost during the entire optimization, which provides a global acceleration of our algorithm.
We demonstrate that our method significantly outperforms previous state-of-the-art approaches in both benchmark and real-world experiments.

\section{Acknowledgments}
This work was supported by the National Science Foundation, award IIS- \#1418511 and CCF-\#1533768, research partnership between Children's Healthcare of Atlanta and the Georgia Institute of Technology, CDC I-SMILE project, Google Faculty Award, Sutter health and UCB.

\clearpage
\bibliography{references}
\bibliographystyle{abbrv}

\clearpage
\appendix
\section{Appendix}
\label{sec:appendix}
\begin{table*}[!ht]
\centering
\caption{Detailed information of the data analytic pipeline constructed for benchmark experiments.}
\label{table:pipeline_setting}
\vspace{5pt}
\begin{tabular}{L{70pt}L{135pt}ccC{55pt}}
\toprule
\textbf{Pipeline step} & \textbf{Algorithm} & \textbf{\#categorical} & \textbf{\#continuous} & \textbf{Total \#hyper-parameters} \\ \midrule
\multirow{4}{*}{\begin{tabular}[c]{@{}l@{}}Feature \\ rescaling\end{tabular}} & Min-max scaler & 0 & 0 & 0 \\
& None & 0 & 0 & 0 \\
& Normalization & 0 & 0 & 0 \\
& Standardization & 0 & 0 & 0 \\ \midrule
\multirow{2}{*}{\begin{tabular}[c]{@{}l@{}}Sample \\ balancing\end{tabular}} & Class weighting & 0 & 0 & 0 \\
& None & 0 & 0 & 0 \\ \midrule
\multirow{13}{*}{\begin{tabular}[c]{@{}l@{}}Feature \\ preprocessing\end{tabular}} & Extremely randomized trees & 2 & 3 & 5 \\
& Fast ICA & 3 & 1 & 4 \\
& Feature agglomeration & 2 & 1 & 3 \\
& Kernel PCA & 1 & 6 & 7 \\
& Random kitchen sinks & 0 & 2 & 2 \\
& Linear SVM & 0 & 2 & 2 \\
& None & 0 & 0 & 0 \\
& Nystroem sampler & 1 & 8 & 9 \\
& PCA & 1 & 1 & 2 \\
& Polynomial combinations & 1 & 2 & 3 \\
& Random trees embedding & 0 & 4 & 4 \\
& Percentile feature selection & 1 & 1 & 2 \\
& Univariate feature selection & 1 & 2 & 3 \\ \midrule
\multirow{14}{*}{Classification} & AdaBoost & 1 & 3 & 4 \\
& Decision tree & 1 & 3 & 4 \\
& Extremely randomized trees & 2 & 3 & 5 \\
& Gaussian naive Bayes & 0 & 0 & 0 \\
& Gradient boosting & 0 & 6 & 6 \\
& K-nearest neighbors & 2 & 1 & 3 \\
& LDA & 1 & 3 & 4 \\
& Linear SVM & 0 & 2 & 2 \\
& Kernel SVM & 2 & 5 & 7 \\
& Multinomial naive Bayes & 1 & 1 & 2 \\
& Passive aggressive & 1 & 2 & 3 \\
& QDA & 0 & 1 & 1 \\
& Random forest & 2 & 3 & 5 \\
& Stochastic gradient descent & 4 & 6 & 10 \\ \midrule
\textbf{Total\#} & 33 & 30 & 72 & 102 \\ \bottomrule
\end{tabular}
\end{table*}
\clearpage

\end{document}